\RequirePackage{fix-cm}
\documentclass[smallextended]{svjour3}       % onecolumn (second format)
\smartqed  % flush right qed marks, e.g. at end of proof
\usepackage{graphicx}
\usepackage{amsmath}
\usepackage{amssymb}
\usepackage{algorithm}
\usepackage{algorithmic}
\usepackage{lscape}
\usepackage[colon]{natbib}

\newtheorem{thm}{Theorem}
\newtheorem{lem}[thm]{Lemma}

\newcommand{\mbf}[1]{\boldsymbol{#1}}

\begin{document}

\title{Managing sparsity, time, and quality of inference in topic models%\thanks{Grants or other notes
}
%\subtitle{Managing sparsity, time, and quality of inference}


\author{Khoat Than \and Tu Bao Ho}


\institute{Khoat Than \and Tu Bao Ho \at
        Japan Advanced Institute of Science and Technology\\
        1-1 Asahidai, Nomi, Ishikawa 923-1292, Japan. \\
              \email{\{khoat, bao\}@jaist.ac.jp}
}


\maketitle

\begin{abstract}
Inference is an integral part of probabilistic topic models, but is often non-trivial to derive an efficient algorithm for a specific model. It is even much more challenging when we want to find a fast inference algorithm which always yields sparse latent representations of documents. In this article, we introduce a simple framework for inference in probabilistic topic models, denoted by \textsf{FW}. This framework is general and flexible enough to be easily adapted to mixture models. It has a linear convergence rate, offers an easy way to incorporate prior knowledge, and provides us an easy way to directly trade off sparsity against quality and time. We demonstrate the goodness and flexibility of \textsf{FW} over existing inference methods by a number of tasks. Finally, we show how inference in topic models with nonconjugate priors can be done efficiently.
\end{abstract}

\section{Introduction} \label{intro}
We are interested in the two important problems in developing probabilistic topic models: \emph{sparsity} and \emph{time}. The sparsity problem is to infer sparse latent representations of documents, while the second problem  asks for an efficient inference algorithm for a topic model. These two problems have been attracting significant interest in recent years, because of their significant impacts and non-trivial nature.

Inference is an integral part of any topic models, and is often NP-hard \citep{SontagR11}. Various methods for efficient inference have been proposed such as folding-in \citep{Hof01}, variational Bayesian  (VB) \citep{BNJ03}, collapsed variational Bayesian (CVB) \citep{TehNW2007collapsed,Asuncion+2009smoothing}, collapsed Gibbs sampling (CGS) \citep{GriffithsS2004}. Sampling-based methods are guaranteed to converge to the underlying distributions, but at a very slow rate. VB and CVB are much faster, and CVB0 \citep{Asuncion+2009smoothing} often performs the best. Although these inference methods are significant developments for topic models, they remain two common limitations that should be further studied in both theory and practice. First, there has been no theoretical upper bound on convergence rate and approximation quality  of inference. Second, the inferred latent representations of documents are extremely dense, which requires huge memory for storage.\footnote{Some attempts have been initiated to speed up inference time and to attack the sparsity problem for Gibbs sampling \citep{MimnoHB12,{YaoMM09}}. Sparsity in those methods does not lie in the latent representations of documents, but lies in sufficient statistics of Gibbs samples. Two main limitations of those methods are that we cannot directly control the sparsity level of sufficient statistics, and that there has been no theory for the goodness of inference and convergence rate. Further, those inference methods are not general and flexible enough to be easily extended to other models such as nonconjugate models.}

Previous researches that have attacked the sparsity problem can be categorized into two main directions. The first direction is probabilistic \citep{WilliamsonWHB2010} for which probability distributions or stochastic processes are employed to control sparsity. The other direction is non-probabilistic for which regularization techniques are employed to induce sparsity \citep{ZhuX2011,ShashankaRS2007,{LarssonU11}}. Although those approaches have gained important successes, they  suffer from some severe drawbacks. Indeed, the probabilistic approach often requires extension of core topic models to be more complex, thus complicating learning and inference. Meanwhile, the non-probabilistic one often changes the objective functions of inference to be non-smooth which complicates doing inference, and requires some more auxiliary parameters associated with regularization terms. Such parameters necessarily require us to do model selection to find an acceptable setting for a given dataset, which is sometimes expensive. Furthermore, a common limitation of these two approaches is that the sparsity level of the latent representations is a priori unpredictable, and cannot be directly controlled.

There is inherently a tension between sparsity and time in the previous inference approaches. Some approaches focusing on speeding up inference \citep{BNJ03,TehNW2007collapsed,Asuncion+2009smoothing} often ignore the sparsity problem. The main reason may be that a zero contribution of a topic to a document is implicitly prohibited in some models, in which Dirichlet distributions \citep{BNJ03} or logistic function \citep{BleiL07} are employed to model latent representations of documents. Meanwhile, the approaches dealing with the sparsity problem often require more time-consuming inference, e.g., \cite{WilliamsonWHB2010,LarssonU11}.\footnote{The model by Zhu and Xing \citep{ZhuX2011} is an exception, for which inference is potentially fast. Nonetheless, their inference method cannot be applied to probabilistic topic models, since unnormalization of latent representations is required.} Note that in many practical applications, e.g., information retrieval and computer vision, fast inference of sparse latent representations of documents is of substantial significance. Hence resolving this tension is necessary.

In this article, we make the contributions as follows:

\begin{itemize}
  \item First, we resolve both problems in a unified way. Particularly, we introduce a simple framework for inference in topic models, called \textsf{FW}, which is general and flexible enough to be easily employed in mixture models. Our framework enjoys the following key theoretical properties: (1) inference converges at a linear rate to the optimal solutions; (2) prior knowledge can be easily incorporated into inference; (3) the sparsity level of latent representations can be directly controlled; (4) it is easy to trade off sparsity against quality and time. We would like to remark that the last two properties are unspecified for existing inference methods.\footnote{Regularization techniques \citep{Tibshirani1996lasso} provide a way to impose sparsity on latent representations, by adding a regularization term to the objective function $f(x)$ to get $g(x) = f(x) + \lambda h(x)$, where $h(x)$ plays a role as a regularization inducing sparsity. Increasing the parameter, $\lambda$, associated with the regularization term may result in sparser solutions. However, it is not always provably true. Further, one cannot a priori decide a desired number of non-zero components of a solution. Hence regularization techniques provide only an indirect control over sparsity. The same holds for the existing probabilistic inference approaches.} %The key to the rationality of our framework is that inference is reformulated as a concave maximization problem over the unit simplex, and hence can be seamlessly solved by sparse approximation algorithms \citep{Clarkson2010,Zhang03,{YuanY12}}.
  \item The second contribution is a theoretical proof for existence of fast inference algorithms with linear convergence rate for many models such as PLSA \citep{Hof01}, CTM \citep{BleiL07}, and mf-CTM \citep{SalomatinYL09}. Interestingly, to the best of our knowledge, this is the first proof for the tractability of inference in nonconjugate models, e.g., CTM, mf-CTM, and tr-mmLDA \citep{PutthividhyaAN10}. Before this work, inference in those nonconjugate models has been believed to be intractable \citep{BleiL07,AhmedX2007,SalomatinYL09}.
\end{itemize}

\textsc{Organization}: after discussing some notations and definitions in Section~\ref{sec: Background}, we introduce the \textsf{FW} framework for inference in Section~\ref{sec: framework}. We also discuss when inference by \textsf{FW} is equivalent to doing ML and MAP inference. Further, we briefly discuss how \textsf{FW} can be applied to PLSA and LDA. The proof of tractability of inference in nonconjugate models is presented in subsection~\ref{sec: nonconjugate}. Section~\ref{sec: experiment} describes our experiments to see practical behaviors of the \textsf{FW} framework.

\section{Notation and definition} \label{sec: Background}

Before going deeply into our framework and analysis, it is necessary to introduce some notations.

\begin{tabular}{ll}
   $\mathcal{V}$: & vocabulary of $V$ terms, often written as $\{1, 2,...,V\}$. \\
   $I_d$: & set of vocabulary indices of the terms appearing in $\boldsymbol{d}$. \\
   $\boldsymbol{d}$: & a document represented as a vector $\boldsymbol{d}= (d_j)_{j \in I_d}$, \\
   & where $d_j$ is the frequency of term $j$ in $\boldsymbol{d}$. \\
  $\mathcal{C}$: & a corpus consisting of $M$ documents, $\mathcal{C} = \{\boldsymbol{d}_1, ..., \boldsymbol{d}_M\}$. \\
  $\boldsymbol{\beta}_k$: & a topic which is a distribution over $\mathcal{V}$. \\
  & $\boldsymbol{\beta}_k = (\beta_{k1},...,\beta_{kV})^t,$ $\beta_{kj} \ge 0, \sum_{j=1}^{V} \beta_{kj} =1$. \\
  $K$: & number of topics. \\
  $\Delta$: & $K$-dimensional unit simplex, $\Delta = \{\lambda \in \mathbb{R}^K: \sum_{k=1}^K \lambda_k = 1, \lambda_k \ge 0 \}$
\end{tabular}

A topic model often assumes that a given corpus is composed from $K$ topics, $\boldsymbol{\beta} = (\boldsymbol{\beta}_1, ..., \boldsymbol{\beta}_K)$, and each document is a mixture of those topics. Example models include PLSA, LDA and many of their variants. Under those models, each document has another latent representation. %Such latent representations of documents can be inferred once those models had been learned previously.

\begin{definition}[Topic proportion]
Consider a topic model $\mathfrak{M}$ with $K$ topics. Each document $\boldsymbol{d}$ will be represented by $\boldsymbol{\theta}=(\theta_1, ..., \theta_K)^t$, where $\theta_k$ indicates the proportion that topic $k$ contributes to $\boldsymbol{d}$, and $\theta_k \ge 0, \sum_{k=1}^K \theta_k = 1$. $\boldsymbol{\theta}$ is called \emph{topic proportion} (or latent representation) of $\boldsymbol{d}$.
\end{definition}

\begin{definition}[ML Inference]
Consider a topic model $\mathfrak{M}$, and a given document $\boldsymbol{d}$. The ML inference problem is to find the  topic proportion $\boldsymbol{\theta}$ that maximizes the likelihood $P(\boldsymbol{d} | \boldsymbol{\theta})$.
\end{definition}

\begin{definition}[MAP Inference]
Consider a topic model $\mathfrak{M}$, and a given document $\boldsymbol{d}$. The MAP inference problem is to find the  topic proportion $\boldsymbol{\theta}$ that maximizes the posterior probability $P(\boldsymbol{\theta} | \boldsymbol{d})$.
\end{definition}

For some applications, it is necessary to infer which topic contributes to a specific emission of a term in a document. Nevertheless, it may be unnecessary for many other applications. Therefore we do not take this problem into account and leave it open for future work.

\section{Framework for fast and sparse inference} \label{sec: framework}

Given a document $\boldsymbol{d}$, we would like to find a desired topic proportion $\boldsymbol{\theta}$ of $\boldsymbol{d}$. The latent representation $\boldsymbol{\theta}$ depends heavily on the objective of inference. The most popular objective is the likelihood of $\boldsymbol{d}$. In many situations, our objective may differ far from the likelihood solely. One example is supervised dimension reduction for which the new representations should be discriminative, i.e, the new representation of a document should remain the most discriminative characteristics of the class to which the document belongs.

To serve various objectives of inference, we propose a novel framework, denoted by \textsf{FW}, which is presented in Algorithm \ref{framework: FW}. Loosely speaking, to do inference for a given document $\boldsymbol{d}$, one first chooses an appropriate objective function $f(\boldsymbol{\theta})$ which is continuously differentiable, concave over the unit simplex $\Delta$. Then one uses a sparse approximation algorithm such as the Frank-Wolfe algorithm \citep{Clarkson2010} to find topic proportion $\boldsymbol{\theta}$. Algorithm \ref{alg: inference} presents in details  the Frank-Wolfe algorithm for inference, where $\boldsymbol{e}_i$'s denote standard unit vectors in $\mathbb{R}^K$. This algorithm follows the greedy approach, and has been proven to converge at a linear rate to the optimal solutions. Moreover, at each iteration, the algorithm finds a provably good approximate solution lying in a face of the simplex $\Delta$.

\begin{thm}\citep{Clarkson2010} \label{thm: 3-1}
Let $f$ be a continuously differentiable, concave function over $\Delta$, and denote $C_f$ be the largest constant so that $f(\alpha\boldsymbol{x}' + (1-\alpha)\boldsymbol{x}) \ge f(\boldsymbol{x}) + \alpha(\boldsymbol{x}' - \boldsymbol{x})^t \nabla f(\boldsymbol{x}) - \alpha^2 C_f, \forall \boldsymbol{x}, \boldsymbol{x}' \in \Delta,$ $\alpha \in [0, 1]$. After $\ell$ iterations, the Frank-Wolfe algorithm finds a point $\boldsymbol{\theta}_{\ell}$ on an $(\ell+1)-$dimensional face of $\Delta$ such that
$
\max_{\boldsymbol{\theta} \in \Delta} f(\boldsymbol{\theta}) - f(\boldsymbol{\theta}_{\ell}) \le {4C_f}/{(\ell +3)}.
$
\end{thm}

%%% Framework FW and inference algorithm
\begin{table}
\begin{minipage}{170pt}
\centering
\begin{algorithm}[H]
   \caption{\textsf{FW} framework}
   \label{framework: FW}
\begin{algorithmic}
   \STATE {\bfseries Input:} document $\boldsymbol{d}$ and topics $\boldsymbol{\beta}_1, ..., \boldsymbol{\beta}_K$.
   \STATE {\bfseries Output:} latent representation $\boldsymbol{\theta}$.
   \STATE {\bfseries Step 1:} select an appropriate objective function $f(\boldsymbol{\theta})$ which is continuously differentiable, concave over $\Delta$.
   \STATE {\bfseries Step 2:} maximize $f(\boldsymbol{\theta})$ over $\Delta$ by the Frank-Wolfe algorithm.
\end{algorithmic}
\end{algorithm}\vspace{0pt}
\end{minipage} \hspace{10pt}
\begin{minipage}{170pt}
\centering
\begin{algorithm}[H]
   \caption{Frank-Wolfe algorithm}
   \label{alg: inference}
\begin{algorithmic}
   \STATE {\bfseries Input: } objective function $f(\boldsymbol{\theta})$.
   \STATE {\bfseries Output:} $\boldsymbol{\theta}$  that maximizes $f(\boldsymbol{\theta})$ over $\Delta$.
   \STATE Pick as $\boldsymbol{\theta}_{0}$ the vertex of $\Delta$ with largest $f$ value.
   \FOR{ $\ell = 0, ..., \infty$}
   \STATE $i' := \arg \max_i   \nabla f(\boldsymbol{\theta}_{\ell})_{i}$;
   \STATE $\alpha' := \arg \max_{\alpha \in [0, 1]} f(\alpha \boldsymbol{e}_{i'} +(1-\alpha)\boldsymbol{\theta}_{\ell})$;
   \STATE $\boldsymbol{\theta}_{\ell +1} := \alpha' \boldsymbol{e}_{i'} +(1-\alpha')\boldsymbol{\theta}_{\ell}$.
   \ENDFOR
\end{algorithmic}
\end{algorithm}
\end{minipage}
\end{table}

It is worth noting some observations about the Frank-Wolfe algorithm:
\begin{itemize}
  \item It achieves a linear rate of convergence, and has provably bounds on goodness of approximate solutions. These are crucial for practical applications;
  \item Overall running time mostly depends on how complicated $f$ and $\nabla f$ are;
  \item It provides an explicit bound on the dimensionality of the face of $\Delta$ on which an approximate solution lies. After $\ell$ iterations, $\boldsymbol{\theta}_{\ell}$ is a convex combination of at most $\ell+1$ vertices of $\Delta$. This implies that we can find an approximate solution to the inference problem which is sparse and provably good;
  \item It is easy to directly control the sparsity level of approximate solutions by trading off sparsity against quality. (Fewer iterations basically results in sparser solutions.)
%  \item[-] It is possible to accelerate the speed, but keep convergence rate of the algorithm. In the algorithm, we have to repeatedly search for auxiliary variable $\alpha$. Those searches can be avoided by selecting a sequence of predefined values as suggested by Clarkson \citep{Clarkson2010}. We can choose $\alpha := 2/(\ell+3)$ at iteration $\ell$. Such choice is capable of maintaining the bound on approximation error (\ref{eq:3-2}), but reducing computations considerably. With this choice, the algorithm works as a stochastic gradient ascent algorithm.
\end{itemize}

We would like to remark that the \textsf{FW} framework is very general and flexible. It can be readily modified in various ways. For example, one can replace the second step by using other approximation algorithms such as sequential greedy approximation \citep{Zhang03} or forward basis selection \citep{YuanY12}. In addition, the first step offers us flexibility to customize objectives of inference.

Perhaps, the most difficult step in our framework is to choose a suitable objective function which can serve our purpose well. Various ways can be considered, however we appeal to the following principle for probabilistic topic models: choosing
\begin{equation}\label{eq:3-1}
f(\boldsymbol{\theta}) = L(\boldsymbol{d} | \boldsymbol{\theta}) + \lambda.h(\boldsymbol{\theta}),
\end{equation}
where $L(\boldsymbol{d} | \boldsymbol{\theta})$ is the log likelihood function of a given document, and $h(\boldsymbol{\theta})$ is a function of the latent representation $\boldsymbol{\theta}$. This principle in turn bears resemblance to regularization techniques \citep{Tibshirani1996lasso} which are widely used for sparse learning. In fact, this principle is implicitly employed in some existing inference methods such as folding-in \citep{Hof01} and VB \citep{BNJ03}, as shown later. We will discuss in details some applications of this principle to PLSA, LDA and other models in the next subsections. The following states some key properties of our framework for inference, which is a corollary of Theorem \ref{thm: 3-1}.

\begin{corollary} \label{corol: 3-1}
Consider a topic model with $K$ topics, and a document $\boldsymbol{d}$. Let $f(\boldsymbol{\theta})$ be continuously differentiable, concave over the simplex $\Delta$. Let $C_f$ be defined as in Theorem \ref{thm: 3-1}. Then inference by \textsf{FW} converges to the optimal solution at a linear rate. In addition, after $\ell$ iterations, the inference error is at most $4C_f/(\ell+3)$, and the topic proportion $\boldsymbol{\theta}$ has at most $\ell+1$ non-zero components.
\end{corollary}

Note that the convergence rate of inference by our framework is linear, i.e., $O(1/\ell)$. It is possible to speed up convergence rate to sub-linear if the Frank-Wolfe algorithm is replaced with forward basis selection \citep{YuanY12}. In addition, if we do not want to work with derivatives $\nabla f$, replacing the Frank-Wolfe algorithm by sequential greedy algorithm \citep{Zhang03} is appropriate. Nonetheless, such extensions are left open for future research. The computational complexity of inference by our framework is exactly that of the Frank-Wolfe algorithm. It heavily depends on how complicated $f$ and $\nabla f$ are.

\subsection{ML and MAP inference} \label{disc: ML and MAP}

Next we would like to discuss two of the most popular inference problems: ML inference where there is no explicit prior over topic proportions; and MAP inference where topic proportions are endowed with a prior distribution. Note that inference for PLSA is ML inference whereas that for LDA and CTM is MAP inference \citep{SontagR11}. We will show how our framework is naturally applicable to ML and MAP inference. Besides, a suitable choice of the objective function implies that inference by the framework is in fact MAP inference.

\begin{lem}\label{lem: 3-1}
Consider a topic model with $K$ topics $\boldsymbol{\beta}_1, ..., \boldsymbol{\beta}_K$, and a given document $\boldsymbol{d}$. The ML inference problem can be reformulated as the following concave maximization problem, over the simplex $\Delta$:
\begin{equation}\label{eq:3-2}
    \boldsymbol{\theta}^* = \arg \max_{\boldsymbol{\theta} \in \Delta} \sum_{j \in I_d} d_j \log \sum_{k=1}^K \theta_k \beta_{kj}.
\end{equation}
\end{lem}

\begin{proof}
Denote by $P(w_j | z_k) = \beta_{kj}$  the probability that the term $w_j$ appears in topic $k$, and by $P(z_k | \boldsymbol{d}) = \theta_k$ the probability that topic $k$ contributes to document $\boldsymbol{d}$. For a given document $\boldsymbol{d}$, the probability that a term $w_j$ appears in $\boldsymbol{d}$ can be expressed as $P(w_j | \boldsymbol{d}) = \sum_{k=1}^K P(w_j | z_k) P(z_k | \boldsymbol{d}) = \sum_{k=1}^K \theta_k \beta_{kj}$. Hence the log likelihood of document $\boldsymbol{d}$ is $\log P(\boldsymbol{d| \boldsymbol{\theta}}) = \log \prod_{j \in I_d} P(w_j | \boldsymbol{d}, \boldsymbol{\theta})^{d_j} = \sum_{j \in I_d} d_j \log P(w_j | \boldsymbol{d}, \boldsymbol{\theta}) = \sum_{j \in I_d} d_j \log \sum_{k=1}^K \theta_k \beta_{kj}$. Note that $\boldsymbol{\theta} \in \Delta$, since  $\sum_{k} \theta_k =1,$ $\theta_k \ge 0, \forall k$. As a result, the inference task is in turn the problem of finding $\boldsymbol{\theta} \in \Delta$ that maximizes the objective function $\sum_{j \in I_d} d_j \log \sum_{k=1}^K \theta_k \beta_{kj}$. \qed
\end{proof}

This lemma tells us that $f(\boldsymbol{\theta}) = \sum_{j \in I_d} d_j \log \sum_{k=1}^K \theta_k \beta_{kj}$ is the objective of ML inference,  which is concave w.r.t $\boldsymbol{\theta}$. So this objective follows the principle (\ref{eq:3-1}). For MAP inference we need an employment of Bayes' rule to see clearly the objective function.

\begin{lem}\label{lem: 3-2}
Consider a topic model with $K$ topics $\boldsymbol{\beta}_1, ..., \boldsymbol{\beta}_K$, in which topic proportions are assumed to be samples of a prior distribution. Assume further that the prior distribution belongs to an exponential family, parameterized by $\alpha$, whose density function can be expressed as $p(\boldsymbol{\theta}| \alpha) \propto \exp(\alpha.t(\boldsymbol{\theta}) - G(\alpha))$. Then the MAP inference problem of a given document $\boldsymbol{d}$ can be reformulated as the problem
\begin{equation}\label{eq:3-3}
    \boldsymbol{\theta}^* = \arg \max_{\boldsymbol{\theta} \in \Delta} \sum_{j \in I_d} d_j \log \sum_{k=1}^K \theta_k \beta_{kj} + \alpha.t(\boldsymbol{\theta}).
\end{equation}
\end{lem}

\begin{proof}
MAP inference is to maximize the posterior probability $P(\boldsymbol{\theta} | \boldsymbol{d})$ given a document $\boldsymbol{d}$. Bayes' rule says that $P(\boldsymbol{\theta} | \boldsymbol{d}) = P(\boldsymbol{d} | \boldsymbol{\theta})P(\boldsymbol{\theta})/P(\boldsymbol{d})$. Hence $\boldsymbol{\theta}^* = \arg \max_{\boldsymbol{\theta} \in \Delta} P(\boldsymbol{\theta} | \boldsymbol{d}) = \arg \max_{\boldsymbol{\theta} \in \Delta} \log P(\boldsymbol{\theta} | \boldsymbol{d}) = \arg \max_{\boldsymbol{\theta} \in \Delta} \log P(\boldsymbol{d} | \boldsymbol{\theta}) + \log P(\boldsymbol{\theta}) = \arg \max_{\boldsymbol{\theta} \in \Delta} \log P(\boldsymbol{d} | \boldsymbol{\theta}) + \alpha.t(\boldsymbol{\theta}) - G(\alpha)$. Ignoring constants and rewriting the likelihood would complete the proof. \qed
\end{proof}

Essentially, this lemma reveals that $f(\boldsymbol{\theta}) = \sum_{j \in I_d} d_j \log \sum_{k=1}^K \theta_k \beta_{kj} + \alpha.t(\boldsymbol{\theta})$ is the objective function of MAP inference, which is exactly of the form (\ref{eq:3-1}), where $t(\mbf{\theta})$ is the sufficient statistics of the prior over $\mbf{\theta}$. However such a function is not always concave. An example is LDA in which $\alpha.t(\boldsymbol{\theta}) = \sum_{k=1}^K (\alpha_k -1) \log \theta_k$ is not concave if $\alpha < 1$, as noted before by \cite{SontagR11}. We next show that with an appropriate choice of the objective function in the form (\ref{eq:3-1}), inference by \textsf{FW} is in fact MAP inference.

\begin{thm}\label{thm: 3-2}
Consider a topic model with $K$ topics, and a document $\boldsymbol{d}$. Let $f(\boldsymbol{\theta}) = L(\boldsymbol{d} | \boldsymbol{\theta}) + \lambda.h(\boldsymbol{\theta})$, where $L(\boldsymbol{d} | \boldsymbol{\theta})$ is the log likelihood of the document, $h(\boldsymbol{\theta})$ is a continuously differentiable, concave function over $\Delta$, $\lambda >0$. Then maximizing $f(\boldsymbol{\theta})$ over $\Delta$ is a MAP inference problem.
\end{thm}

\begin{proof}
Consider the marginal distribution of the random variable $\boldsymbol{\theta}$ whose density function is of the form $p(\boldsymbol{\theta} | \lambda) \propto \exp(\lambda.h(\boldsymbol{\theta}))$. Then  $\boldsymbol{\theta}^* = \arg \max_{\boldsymbol{\theta} \in \Delta} P(\boldsymbol{\theta} | \boldsymbol{d}) = \arg \max_{\boldsymbol{\theta} \in \Delta} \log P(\boldsymbol{\theta} | \boldsymbol{d}) = \arg \max_{\boldsymbol{\theta} \in \Delta} \log P(\boldsymbol{d} | \boldsymbol{\theta}) + \log P(\boldsymbol{\theta}| \lambda) = \arg \max_{\boldsymbol{\theta} \in \Delta} \log P(\boldsymbol{d} | \boldsymbol{\theta}) + \lambda.h(\boldsymbol{\theta})$. The objective of this optimization problem is exactly the function $f(\boldsymbol{\theta})$, completing the proof. \qed
\end{proof}

\subsection{Application to PLSA and LDA} \label{sec: application}

We now discussed how \textsf{FW} can be adapted to the two of the most influential topic models, PLSA \citep{Hof01} and LDA \citep{BNJ03}. Lemma \ref{lem: 3-1} provides us a connection between ML inference and concave optimization. As a consequence, inference in PLSA can be reformulated as an \emph{easy} optimization problem, and can be seamlessly resolved by \textsf{FW}. Combining this with Corollary~\ref{corol: 3-1}, we obtain the following.

\begin{corollary} \label{corol: 3-2}
Consider PLSA with $K$ topics, and a document $\boldsymbol{d}$. Then there exists an algorithm for inference that converges to the optimal solution at a linear rate, and that allows us to efficiently find a sparse topic proportion $\boldsymbol{\theta}$ with a guaranteed bound on inference error.
\end{corollary}

Note that according to Lemma \ref{lem: 3-1}, the objective function of inference in PLSA is $f(\boldsymbol{\theta}) = \sum_{j \in I_d} d_j \log \sum_{k=1}^K \theta_k \beta_{kj}$. This objective turns out to be of the form (\ref{eq:3-1}) where $h(\boldsymbol{\theta}) \equiv 0$. It is easy to check that this function is continuously differentiable, concave over the simplex $\Delta$ if $\boldsymbol{\beta} > 0$. Hence, the Frank-Wolfe algorithm can be exploited for inference. One can handily do MAP inference for PLSA by  modifying the objective function to be of the form (\ref{eq:3-1}). While MAP inference for PLSA has been studied by \cite{ShashankaRS2007} and \cite{LarssonU11}, their methods result in concave-convex objective functions and thus have no guaranteed bound for convergence.

We next turn our consideration to LDA \citep{BNJ03}. It is known \citep{SontagR11} that finding a topic proportion for a given document in LDA is an MAP inference problem, where the objective function is $f(\boldsymbol{x}) = \sum_{j \in I_d} d_j \log \sum_{k=1}^K \theta_k \beta_{kj} + \sum_{k=1}^K (\alpha_k -1) \log \theta_k$. This objective is of the same form with (\ref{eq:3-1}), where $h(\boldsymbol{\theta}) = (\log \theta_1, ..., \log \theta_K)^t$ and $\lambda = (\alpha_1 -1, ..., \alpha_K -1)$. $h(\boldsymbol{\theta})$ and $\lambda$ originally come from the Dirichlet prior over topic proportions. One can interpret $\lambda.h(\boldsymbol{\theta})$ to be a regularization term which induces \emph{sparse} solutions for $\lambda < 1$. However, such a regularization does not always result in a concave objective function, and hence causes the inference in LDA to be NP-hard  \citep{SontagR11}. Furthermore, such a regularization requires all topics to have non-zero contributions to a specific document, since the function $\log \theta_k$ requires $\theta_k > 0$ to be well-defined. Hence, LDA cannot infer latent representations which are sparse in common sense.

To find sparse latent representations in LDA, some modifications are necessary. One can readily apply the \textsf{FW} framework to LDA where the objective is the log likelihood function. Other employments of the \textsf{FW} framework can yield MAP inference for LDA as suggested by Theorem \ref{thm: 3-2}. In those cases, it amounts to endowing new priors other than Dirichlet over topic proportions.

\subsection{Topic models with nonconjugate priors}\label{sec: nonconjugate}

Many practical tasks naturally require that topic proportions should follow some other priors than Dirichlet. Those tasks lead to the use of nonconjugate priors over $\boldsymbol{\theta}$. A typical example is the use of logistic normal distributions to model correlations between topics \citep{BleiL07,SalomatinYL09,{PutthividhyaAN10}}. As noted by various researchers, non-conjugacy of priors causes significant difficulties for deriving good inference/learning algorithms. As a consequence, existing inference methods \citep{BleiL07,SalomatinYL09,PutthividhyaAN10,AhmedX2007} are often slow, and do not have any guarantee on neither convergence rate nor inference quality. On the contrary, we will show that inference in many nonconjugate models can be done efficiently. To substantiate this claim, we study \emph{correlated topic models} (CTM) by \cite{BleiL07}.

The main objective of CTM is to uncover relationships between hidden topics. \cite{BleiL07} employ the normal distribution $\mathcal{N}(\mbf{x}; \mbf{\mu, \Sigma})$ with mean $\mbf{\mu}$ and covariance matrix $\mbf{\Sigma}$ to model those relationships. Topic proportions are computed by the logistic transformation as $\theta_k = e^{x_k}/\sum_{j=1}^K e^{x_j}$. Since such a transformation maps a $K$ dimensional vector to a $(K-1)$ dimensional vector, various $\mbf{x}$'s can correspond to a single vector $\mbf{\theta}$. Therefore, for identifiability, we can use transformation $x_k = \log \theta_k$ to recover $\mbf{x}$ from $\mbf{\theta}$ without loss of generality.

A key to our arguments is the observation that $\mathcal{N}(\mbf{x}; \mbf{0, \Sigma})$ is sufficient to model correlations between topics. The reasons come from noticing that we are mostly interested in the covariance matrix $\mbf{\Sigma}$, and that the covariance is invariant w.r.t change in $\mbf{\mu}$ because of $\mbf{\Sigma} = cov(\mbf{x}) = cov(\mbf{x+a})$ for any $\mbf{a}$. Note that using $\mathcal{N}(\mbf{x}; \mbf{0, \Sigma})$ should be much less complicated than using $\mathcal{N}(\mbf{x}; \mbf{\mu, \Sigma})$ to model correlations. More importantly, inference in this case would be easy as shown below.

\begin{thm} \label{thm: 3.3.1}
Consider CTM with $K$ topics for which $\mathcal{N}(\mbf{x}; \mbf{0, \Sigma})$ models correlations between hidden topics, and a document $\boldsymbol{d}$. Assume further that the transformation $x_k = \log \theta_k$ is used to recover $\mbf{x}$ from topic proportion $\mbf{\theta}$ of $\mbf{d}$. Then there exists an algorithm for MAP inference of $\mbf{\theta}$ that converges to the optimal solution at a linear rate.
\end{thm}

\begin{proof}
Note that $p(\mbf{x}; \mbf{0, \Sigma}) = \frac{1}{\sqrt{\det(2\pi \mbf{\Sigma})}} \exp(-\frac{1}{2} \mbf{x}^t \mbf{\Sigma}^{-1} \mbf{x})$ is the density function of $\mathcal{N}(\mbf{x}; \mbf{0, \Sigma})$. From Lemma \ref{lem: 3-2}, the MAP inference problem in CTM can be reformulated as, where $\log \mbf{\theta} = (\log \theta_1, ..., \log \theta_K)^t$,

\begin{equation}\label{eq:2}
    \boldsymbol{\theta}^* = \arg \max_{\boldsymbol{\theta} \in \Delta} \sum_{j \in I_d} d_j \log \sum_{k=1}^K \theta_k \beta_{kj} - \frac{1}{2}(\log \mbf{\theta})^t \mbf{\Sigma}^{-1} \log \mbf{\theta}.
\end{equation}

We next show that the objective function of this problem is concave over the unit simplex $\Delta$. Indeed, it is easy to check that the term $\sum_{j \in I_d} d_j \log \sum_{k=1}^K \theta_k \beta_{kj}$ is concave w.r.t $\mbf{\theta}$. Our remaining task is to show the concavity of the term $y(\mbf{\theta}) = - \frac{1}{2}(\log \mbf{\theta})^t \mbf{\Sigma}^{-1} \log \mbf{\theta}$. Its first and second derivatives are
\begin{eqnarray*}
y' &=& -diag\left(\frac{1}{\mbf{\theta}}\right) \mbf{\Sigma}^{-1} \log \mbf{\theta}, \\
y'' &=& -diag\left(\frac{1}{\mbf{\theta}}\right) \left[ \mbf{\Sigma}^{-1} - diag\left(\mbf{\Sigma}^{-1} \log \mbf{\theta}\right) \right] diag\left(\frac{1}{\mbf{\theta}}\right),
\end{eqnarray*}
where  $diag(1/\mbf{\theta})$ is the diagonal matrix of size $K$ whose diagonal elements are $\frac{1}{\theta_1}, ..., \frac{1}{\theta_K}$, respectively.

Note that $diag(1/\mbf{\theta})$ is positive definite for any feasible solution $\mbf{\theta} \in \Delta$. One can easily check the fact that a diagonal matrix is negative semidefinite iff all of its diagonal elements are not positive. Note further that $\mbf{\Sigma}^{-1} \log \mbf{\theta} \le 0$, due to $0 \le \mbf{\theta} \le 1$ and positive definiteness of $\mbf{\Sigma}$. As a result, $diag\left(\mbf{\Sigma}^{-1} \log \mbf{\theta}\right)$ is negative semidefinite. Combining it with the positive definiteness of $\mbf{\Sigma}^{-1}$, we can conclude that $y''$ is negative definite for each feasible solution $\mbf{\theta}$ in $\Delta$. This implies that $y(\mbf{\theta})$ is a concave function over the interior of $\Delta$. As a consequence,  (\ref{eq:2}) is a  concave maximization problem over the simplex.

Even though (\ref{eq:2}) is a concave maximization problem, the objective function is not specified on the boundary of $\Delta$. Hence, the \textsf{FW} algorithm cannot be directly applied. Fortunately, algorithms by \cite{Jaggi11} work well in the interior of $\Delta$ and have a linear rate of convergence.
\qed
\end{proof}

This theorem basically says that MAP inference in CTM is in fact tractable and  can be done very fast, which is contrary to the existing belief in the topic modeling literature. Moreover, the inference quality is guaranteed to be good. We believe that the same results can be derived for many other models such as those by \cite{SalomatinYL09,PutthividhyaAN10,VirtanenJKD12}. It is worthwhile noting that optimal solutions to the MAP inference problem in CTM are no longer sparse, because $\mbf{\theta}$ would not to be optimal if it contains any zero component.

If one insists on using the normal distribution in the full form to model correlations, some slight modifications are sufficient to do MAP inference efficiently. Indeed, using similar arguments as in the proof above, we can show that the objective function of inference is concave over the convex region $\{\mbf{\theta} \in \Delta : \log \theta_k \le \mu_k, \forall k\}$. This observation implies that inference is in fact a concave maximization problem over a closed convex set. Hence, there exists an efficient algorithm for inference.

\begin{thm} \label{thm: 3.3.2}
Consider CTM with $K$ topics for which $\mathcal{N}(\mbf{x}; \mbf{\mu, \Sigma})$ models correlations between hidden topics, and a document $\boldsymbol{d}$. Assume further that the transformation $x_k = \log \theta_k$ is used to recover $\mbf{x}$ from topic proportion $\mbf{\theta}$ of $\mbf{d}$. Then there exists an algorithm for MAP inference of $\mbf{\theta} \in \{\mbf{\theta}' \in \Delta : \log \theta'_k \le \mu_k, \forall k\}$ that converges to the optimal solution at a linear rate.
\end{thm}

\remark{We have seen that \textsf{FW} cannot be used directly to do inference for CTM, since the objective function of inference (\ref{eq:2}) is not well-defined on the boundary of the unit simplex. However, we may do inference for CTM by \textsf{FW} with some slight modifications. Indeed, one can replace the initial step of the Frank-Wolfe algorithm by setting $\mbf{\theta}_0$ to be $(1/K,...,1/K)^t$ or a certain point in the interior of $\Delta$. We believe that this slight modification does not change significantly the convergence rate of the original algorithm.}

\remark{Once topic proportions can be inferred efficiently, we can easily design a new learning algorithm for CTM. One can forget the latent variable $z$ and just do MAP inference to find $\mbf{\theta}$ for each document in the E-step. The M-step  maximizes the likelihood of the training data w.r.t. the model parameters. The same idea was investigated by \cite{ThanH2012fstm}, resulting in a topic model with many attractive properties for dealing with large data. We believe that if following such a learning approach, we can easily learn CTM at a large scale, and hence enable large-scale analyses of correlations of latent topics.}

\section{Empirical evaluation} \label{sec: experiment}

In this section, we explore how well our framework works compared with existing inference methods. We first investigate some fundamental characteristics of the \textsf{FW} framework, including sparsity of the inferred topic proportions, inference time, and inference quality. In addition to theoretical analysis and demonstration, we made a library for use in practice that is very easy for researchers/users to incorporate our framework into their customized models, just by writing their own objective functions. This may help substantially reduce complication and time for researchers when designing new topic models. The library is general enough to be applicable to inference in other literatures than topic modeling.\footnote{The library is freely available at www.jaist.ac.jp/$\sim$s1060203/codes/FW/.}

The flexibility of the \textsf{FW} framework is evidenced by two specific applications. In the first one, we successfully develop \emph{fully sparse topic models} (FSTM) \citep{ThanH2012fstm} which is a simplified variant of PLSA and LDA. FSTM has been demonstrated to work well and has various attractive properties for dealing with large data. In the second application, we employ \textsf{FW} to design effective methods for supervised dimension reduction \citep{ThanHNP12}.

%We report some experiments and comparisons in the supplementary material. Those include comparison of accurate and reasonable inference of topics, comparison with the specialized inference method for STC \citep{ZhuX2011} which is dedicated to sparsity. Those comparisons show that inference by our framework is more reasonable and consistently achieves much sparser solutions than other methods, while enjoying competitive inference time.

\subsection{Time, sparsity, and quality}

Analyses in the previous section have shown that inference by our framework is both fast and provably good, if provided a suitable choice of the objective function. In this section, we demonstrate empirically that even with the modest choice, say likelihood, our framework infers comparably well. Three inference methods were taken in comparison: Folding-in \citep{Hof01},  Variational Bayesian \citep{BNJ03}, denoted by VB, and \textsf{FW}.\footnote{CVB, CVB0, and CGS were not included for some reasons. CVB is often slower than VB \citep{MukherjeeB2009}; CVB0 is faster than VB but works on documents which are not in bag-of-words representation; CGS is often slowest. Futhermore, these methods can achieve comparable quality as long as suitable parameter settings are chosen \citep{Asuncion+2009smoothing}. Hence VB is selected to be a representative.}  The objective function for \textsf{FW}  is  the log likelihood function. Five corpora were used in the investigation, of which some statistics are shown in Table \ref{Table: Data-info}.\footnote{AP was retrieved from http://www.cs.princeton.edu/$\sim$blei/lda-c/ap.tgz. KOS, NIPS, and Enron were from http://archive.ics.uci.edu/ml/datasets/. Grolier was from http://cs.nyu.edu/$\sim$roweis/data.html} For each corpus, we first trained the LDA model on the training part. We then did inference on the test set with the same criteria of convergence.\footnote{At most 1000 iterations are allowed for inference, and the algorithm will converge if the relative change of the objective is less than $10^{-6}$.}

\begin{table}[tp]
\centering
\caption{Data for experiments.}
\label{Table: Data-info}
\begin{tabular}{lllll}
\hline
Data & Training size & Testing size & \#Terms & \#Classes \\
\hline
AP & 2021 & 225 & 10473 & 0 \\
KOS & 3087 & 343 & 6906 & 0 \\
NIPS & 1350 & 150 & 12419 & 0 \\
Grolier & 23044 & 6718 & 15276 & 0 \\
Enron & 35875 & 3986 & 28102 & 0 \\
20Newsgroups & 15935 & 3993 & 62061 & 20 \\
Emailspam & 3461 & 866 & 38729 & 2 \\
\hline
\end{tabular}
\vspace{-10pt}
\end{table}

\emph{Inference time:} the first measure for comparison is inference time. Figure~\ref{fig: FW-plsa-lda} depicts the results of inference on 5 corpora. We observe that Folding-in did slowest. VB did much more quickly than Folding-in. Each iteration of Folding-in took very few computations, much less than that of VB. However, VB often reached convergence in much less steps than Folding-in. That is why overall VB did more quickly. Compared with Folding-in and VB, our framework did inference significantly faster. \textsf{FW} often reached convergence in a few tens of iterations. Note that complexity of our framework heavily depends on how complicated the objective is. In this case, the objective is the log likelihood which needs few computations to be evaluated.  One can realize that  the inference time of \textsf{FW} was not quickly scaled up as the number of topics $K$ increases, while VB and Folding-in increased much faster. This suggests that our framework is substantially more scalable than Folding-in and VB.

\begin{figure}[tp]
\centering
  \includegraphics[width=\textwidth]{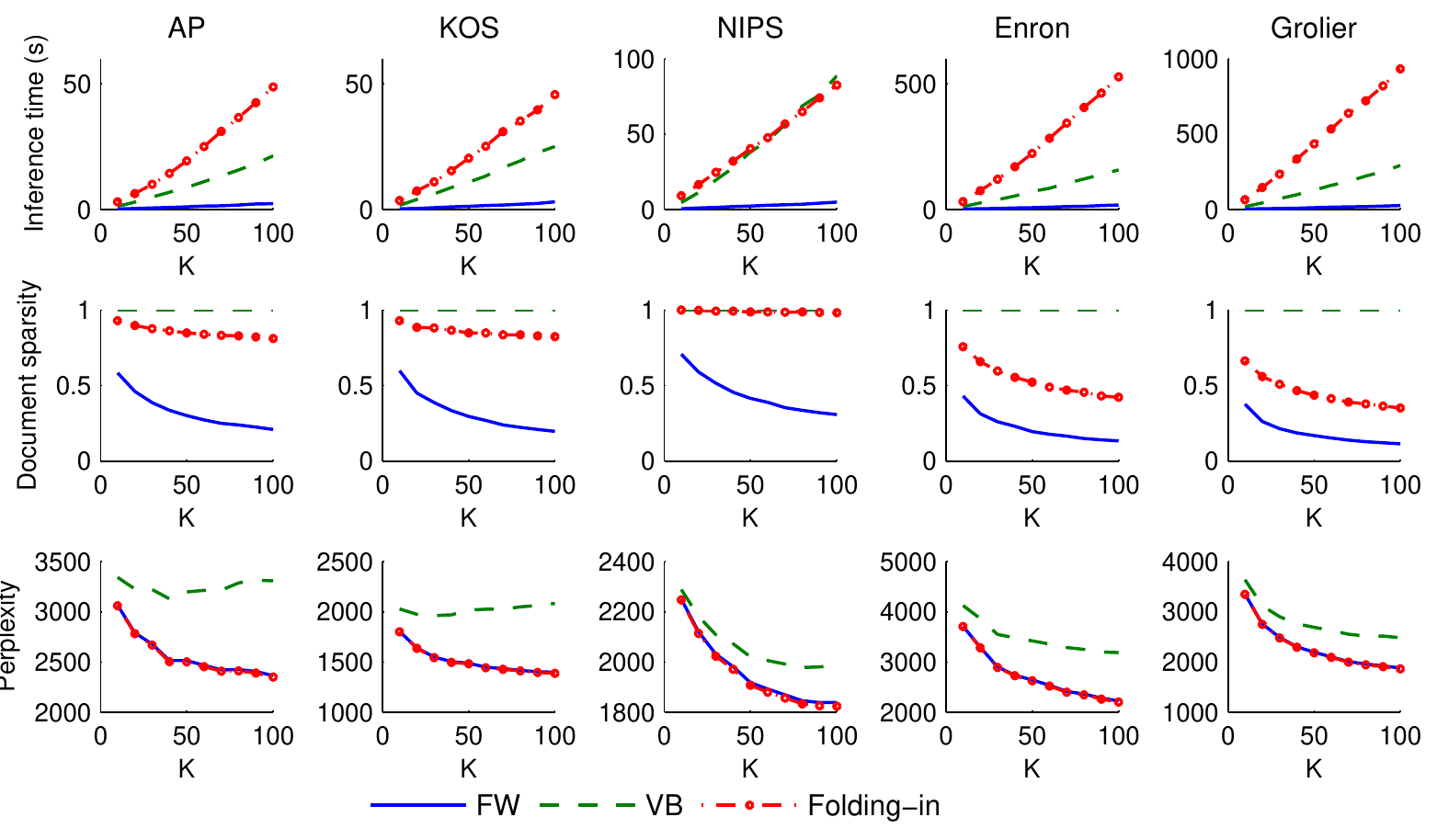}\\
%  \vspace{-10pt}
  \caption{Comparison of inference methods as the number of topics increases. Lower is better.}
  \label{fig: FW-plsa-lda}
  \vspace{-5pt}
\end{figure}

\emph{Document sparsity:} we next consider how sparse the inferred topic proportions are. Sparsity of a given document is the fraction of nonzero elements in the inferred latent representation. It is averaged for each test set, and is depicted in the second row of Figure \ref{fig: FW-plsa-lda}. Note that inference by our framework always found very sparse topic proportions. The sparsity level increases as we model with more topics. Surprisingly, inference by Folding-in sometimes achieves sparse topic proportions. One possible reason is that Folding-in may inherit sparsity of original data, since inference by Folding-in simply does addition and multiplication on sparse data. Nevertheless, it is not always for Folding-in to achieve sparse solutions without a principled mechanism. Unsurprisingly, VB did not find any sparse latent representations of documents.

\emph{Perplexity:} Corollary \ref{corol: 3-1} suggests that inference by our framework theoretically finds provably good solutions. This theoretical result is further supported by experiments. The last row of Figure \ref{fig: FW-plsa-lda} shows the goodness of different inference methods in terms of perplexity \citep{BNJ03,{BleiL07}}. Loosely speaking, perplexity is the inverse of the geometric mean of the probabilities of words appearing in the testing documents, and is calculated on the testing set $\mathcal{D}$ by $Perplexity(\mathcal{D}) = \exp\left( - {\sum_{\boldsymbol{d} \in \mathcal{D}} \log P(\boldsymbol{d})} / {\sum_{\boldsymbol{d} \in \mathcal{D}} ||\boldsymbol{d}||_1}\right).$ Observing Figure \ref{fig: FW-plsa-lda}, we see that Folding-in and \textsf{FW} achieved comparably good predictive power. They performed much better than VB even though they were given the same models which had been trained before.

To explain this phenomenon, more thorough investigations are necessary. We observed that in all cases, LDA learned very small parameters $\alpha$ of the Dirichlet priors. Remember that when $\alpha < 1$, inference in LDA is NP-hard \citep{SontagR11}. The NP-hardness may prevent the variational method from quickly inferring good solutions. This may be the main reason for the inferior performance of VB. Note further that inference in LDA is MAP inference, whose objective is different from the likelihood of data. But perplexity mainly relates to likelihood. Therefore, asynchronous objective functions for inference is another reason for inferior performance of VB in terms of perplexity.

\emph{Separability of documents in the topical space:} topic models are often expected to provide us a soft clustering of documents in the space of topics, i.e., clustering documents into topical clusters. Hence we would like to see how well inference methods cluster the testing documents. A good method should cluster documents into topics \emph{separately}. In other words, in the topical space, the documents should be separately clustered. To see this, we use the inferred latent representations of documents, and visualize the first 3 dimensions. Figure \ref{fig: separability} shows the distribution of documents in the topical space. One can observe that the documents projected by VB spread around the axes, and they were not separated clearly into clusters. Similar phenomenon can be observed for Folding-in. Meanwhile, when projected by \textsf{FW}, each document focused more on few topics, and the documents were separated into clusters explicitly. We observed that inference by our framework often places very high probability on one topic, small probabilities on few more topics, and zero on others. This may be why, in the topical space, the documents are explicitly clustered. As a result, inference by our framework provides a better clustering of documents in the topical space.

\begin{figure}[tp]
\centering
  \includegraphics[width=\textwidth]{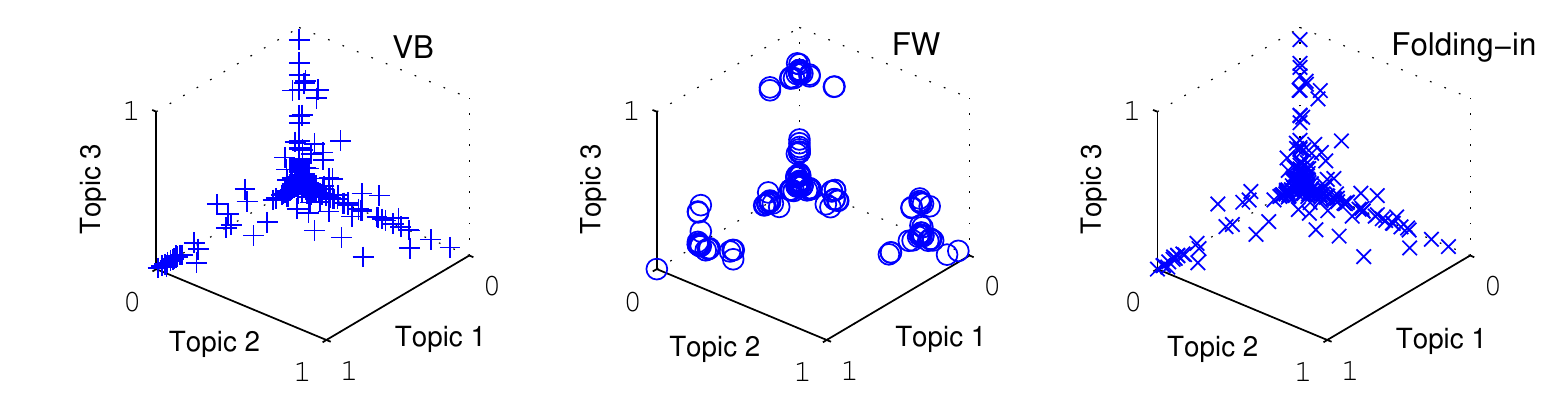}\\
%  \vspace{-10pt}
  \caption{Separability of documents in the space of topics, inferred by different methods on AP with $K=10$. Folding-in and VB do not provide  separate clusters of documents. Meanwhile, \textsf{FW} always separates documents explicitly into clusters associated with latent topics.}
  \label{fig: separability}
%  \vspace{-10pt}
\end{figure}

\subsection{Convergence rate and trade-off} \label{sec: trade-off}
When facing with large-scale settings including large corpora, extremely high dimensionality, and large number of topics, fast algorithms and compact storage demands are highly desired.  Hence a principled way to trade off quality against time and storage requirement is sometimes necessary. Fortunately, the Frank-Wolfe algorithm can fulfill those desires for not only topic modeling but also other literatures. Indeed, it is provably fast and provides a simple way to decide the sparsity level of solutions, just by limiting the number of iterations.

We investigated further how quick \textsf{FW} reaches convergence in practice. The experiments were done with AP (small size) and Enron (average size), and on the learned LDA with $K=100$ topics. Results are shown in Figure~\ref{fig: trade-off}. One can realize that \textsf{FW} reached convergence very quickly. We found that in most cases, after 20 iterations on average the quality was almost stable. Note that the dimension of the inference problem is $K=100$ which is much larger than 20. The sparsity level of solutions got stable almost after 30 iterations. The same phenomenon was observed on other corpora. These facts suggest that \textsf{FW} can converge very quickly in practice despite of the loose bound in Theorem \ref{thm: 3-1}. This property is attractive for practical applications.

\begin{figure}[tp]
\centering
  \includegraphics[width=\textwidth]{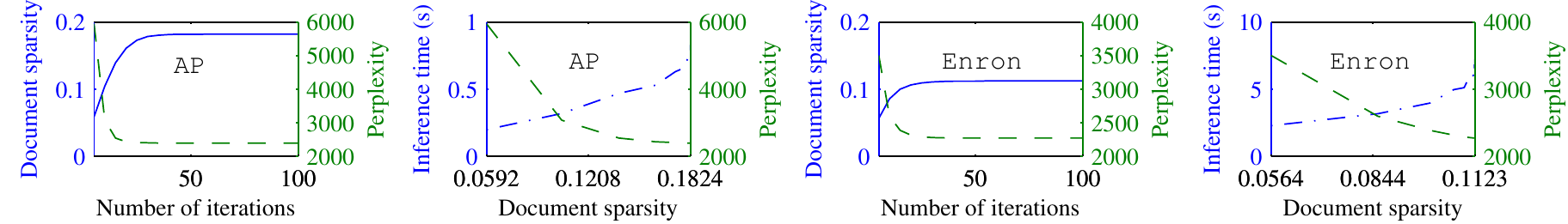}\\
%  \vspace{-10pt}
  \caption{Illustration of trading off sparsity against time and quality. \textsf{FW} is able to reach convergence very quickly. After 20 iterations on average, its quality in terms of perplexity was almost stable, even though the number of topics is much larger ($K=100$).}
  \label{fig: trade-off}
%  \vspace{-10pt}
\end{figure}

\section{Conclusion}
We make two contributions in this article. First, a framework (\textsf{FW}) for efficiently inferring sparse latent representations of documents is introduced. From theoretical and empirical analyses, the framework is shown to work significantly fast and always infer sparse solutions. Second, we show that inference in topic models with nonconjugate priors can be done efficiently, which is contrary to the previous belief \citep{BleiL07,AhmedX2007,SalomatinYL09,PutthividhyaAN10} that inference in nonconjugate models is intractable.

%\begin{acknowledgements}
%\end{acknowledgements}

\bibliographystyle{plainnat}
\bibliography{topic-models,other}


\end{document}